\documentclass{article}

\PassOptionsToPackage{numbers, compress}{natbib}

\usepackage[preprint]{style/neurips_2023}

\usepackage[utf8]{inputenc} 
\usepackage[T1]{fontenc}    
\usepackage{hyperref}       
\usepackage{url}            
\usepackage{booktabs}       
\usepackage{amsfonts}       
\usepackage{nicefrac}       
\usepackage{microtype}      
\usepackage{xcolor}         
\usepackage{listings}
\usepackage{adjustbox}
\usepackage{multirow}
\usepackage{multicol}
\usepackage{caption}
\usepackage{textcomp}

\usepackage{amsmath}
\DeclareMathOperator*{\argmax}{arg\,max}

\usepackage{amsthm,amsmath,amsfonts,amssymb}
\usepackage{bbm}

\usepackage{bm}
\DeclareMathAlphabet{\mathsfit}{\encodingdefault}{\sfdefault}{m}{sl}
\SetMathAlphabet{\mathsfit}{bold}{\encodingdefault}{\sfdefault}{bx}{n}

\newtheorem{theorem}{Theorem}[section]
\newtheorem{lemma}[theorem]{Lemma}
\newtheorem{prop}[theorem]{Proposition}

\newtheorem{definition}{Definition}[section]

\usepackage{algorithm,algpseudocodex}

\title{Inferring dynamic regulatory interaction graphs from time series data with perturbations}

\author{%
Dhananjay Bhaskar\\
Yale University\\
\texttt{dhananjay.bhaskar@yale.edu} \\
\And
Sumner Magruder \\
Yale University \\
\texttt{sumner.magruder@yale.edu} \\
\And
Edward De Brouwer \\
Yale University \\
\texttt{edward.debrouwer@yale.edu} \\
\And
Aarthi Venkat \\
Yale University \\
\texttt{aarthi.venkat@yale.edu} \\
\And
Frederik Wenkel \\
Mila, Université de Montréal \\
\texttt{frederik.wenkel@umontreal.ca} \\
\And
Guy Wolf \\
Mila, Université de Montréal \\
\texttt{wolfguy@mila.quebec} \\
\And
Smita Krishnaswamy\footnote{Corresponding Author} \\
Yale University \\
\texttt{smita.krishnaswamy@yale.edu}
}

\begin{document}

\maketitle

\begin{abstract}

Complex systems are characterized by intricate interactions between entities that evolve dynamically over time. Accurate inference of these dynamic relationships is crucial for understanding and predicting system behavior. In this paper, we propose \emph{Regulatory Temporal Interaction Network Inference (RiTINI)} for inferring time-varying interaction graphs in complex systems using a novel combination of space-and-time graph attentions and graph neural ordinary differential equations (ODEs). RiTINI leverages time-lapse signals on a graph prior, as well as perturbations of signals at various nodes in order to effectively capture the dynamics of the underlying system. This approach is distinct from traditional causal inference networks, which are limited to inferring acyclic and static graphs. In contrast, RiTINI can infer cyclic, directed, and time-varying graphs, providing a more comprehensive and accurate representation of complex systems. The graph attention mechanism in RiTINI allows the model to adaptively focus on the most relevant interactions in time and space, while the graph neural ODEs enable continuous-time modeling of the system's dynamics. We evaluate RiTINI's performance on various simulated and real-world datasets, demonstrating its state-of-the-art capability in inferring interaction graphs compared to previous methods.

\end{abstract}

\section{Introduction}

Biophysical and biochemical systems are highly complex and dynamic entities whose behavior is governed by interactions between regulatory elements. For instance, cells contain numerous components such as genes, proteins, chromatin, small molecules and structural elements. The phenotypic identity and behavior of the cell is controlled largely by the proteins expressed in the cell, which are in turn determined by the gene regulatory network which modulates gene expression. Similarly, in the brain, neural activity is modulated by networks of neurons which stimulate and repress one another based on complex and dynamic connectivity patterns. In each of these cases, knowledge of the interaction graph and its dynamics could help simulate and analyze the behavior of the system.  Here, we focus on the problem of interaction graph inference in time series data, and propose a graph ordinary differential equation (graph-ODE) based method  equipped with dual attention mechanisms for tackling this problem.  We can generally perform such inference on either time-trace measurements which are available as neuronal activity readouts, or inferred time traces from single cell data using software such as TrajectoryNet \cite{pmlr-v119-tong20a}.

We define interaction graph inference as the problem of inferring a directed (though not necessarily acyclic) time-varying graph $G(t)=(\mathcal{V}, \mathcal{E}(t))$ on a fixed set of vertices $\mathcal{V} = \{1,2,\cdots, N\}$, with dynamic edge weights $W_{ij}(t) > 0$ if $(i,j) \in \mathcal{E}(t)$. Features $X(v,t)$ on vertex $v \in \mathcal{V}$ have dynamic expression based on the structure of $G(t)$, i.e. $\partial_t X(v,t)$ is a function of $X(v,t)$ and ${X(\mathcal{N}(v), t-\delta)}$, i.e., the features on neighbors of $v$ at a time $t-\delta$. In other words, the expression of feature $X$ on vertex $v$ is determined by the interaction between $v$ and its directed neighbors (or regulators) $\mathcal{N}(v)$. In addition, some systems (such as neurons) exhibit hysteresis or memory, in that the interaction is mediated not just by the value of the neighbors at one timepoint $t-\delta$, but rather the value over a time interval, $\delta \in [t, t-\tau]$. 
In the absence of prior knowledge, $G(t)$ would ideally be inferred as the graph that best explains the features $X(v,t)$, i.e., with the lowest error. However, this is an inverse problem where many graphs can plausibly explain these features. To break the symmetry, we make use of prior knowledge and regularize this graph  with a graph prior $\mathcal{P}(t) = (\mathcal{V}_{\mathcal{P}}, \mathcal{E}_{\mathcal{P}}(t))$ such that deviation from the prior graph is penalized according to the discrepancy between the adjacency matrices $\| W(t) - \mathcal{E}_{\mathcal{P}}(t) \|_{F}$.  In contrast to typical causal discovery frameworks, we allow the graphs to contain cycles.

Our method operates by using the prior graph $\mathcal{P}(0)$ as an initial condition, and integrating a learnable graph-ODE over time with an MLP aggregation function  $f_\theta$, i.e. $\partial_t X(v,t) = f_\theta(v, \mathcal{N}(v), \alpha, t)$. We are able to learn how  nodes interactions change strengths over time using a novel attention scheme that combines a vertex-neighborhood attention parameter $\alpha(t)$ and temporal attention parameters $l(t)$. The temporal attention parameters allow us to characterize the presence of hysteresis in the system, where the variance of the resulting distribution indicates hysteresis. The dynamic graph inference is made by utilizing the the readouts of the spatial attention parameter $\alpha(t)$ over time.  

We show that the quality of our predictions improve by training on temporally localized perturbations on vertex features, similar to what can be done in biological systems using optogenetic stimulation or gene expression perturbations. In such cases the value of a vertex feature can be manipulated in a localized time frame, and this effect then propagates through the network in a manner determined by network structure. We show that this error propagation helps infer the connectivity of the network. We showcase results on neuronal simulation data obtained using NEST, gene regulatory simulation data from SERGIO, and scRNA-seq data from human embryonic stem cell differentiation. 


Our \textbf{key contributions} are as follows: 

\begin{itemize}
    \item We define the problem of interaction graph inference and outline its relevance to complex biological systems. 
    \item We propose a new method, Regulatory Temporal Interaction Network Inference (RiTINI), to learn the underlying interaction graph from a multivariate time series dataset using a novel attention based graph ODE network. Our method operates in continuous time and can thus allows us to predict, and interpolate the trajectories continuously.
    \item We formulate new variants of graph attention over time and graph vector space for the purpose of learning interaction graphs. 
    \item We further show that this method can be trained with node perturbations and that prediction of their effects can improve graph structure learning. 
    \item We show that our method is competitive in a variety of real-world problems such as neurology (recovery of neuronal networks) and molecular biology (recovery of gene regulatory networks).
\end{itemize}



\section{Related Works}

\subsection{Inferring Functional Connectivity in the Brain}


Brain functional connectivity is a prominent and pivotal field of study that seeks to infer interaction graphs from time series data, thereby illuminating the underlying neural dynamics and communication patterns within the brain. 
Numerous approaches have been proposed to tackle this complex task, among which Granger causality stands out as one of the most widely employed. Granger causality has been extensively utilized to unveil causal relationships between different brain regions, leveraging their temporal dynamics \cite{granger_causality_1969, friston_granger_2013, zhang_dynamic_2019, sethi_directed_2017, ribeiro_granger_2021}. For instance, Zhang et al. conducted an analysis utilizing Granger causality to investigate the dynamic functional connectivity associated with eating behaviors. Their study underscored the advantages of capturing time-varying interactions between brain regions and elucidated network dynamics that correlated with distinct aspects of eating behavior \cite{zhang_dynamic_2019}. 

\subsection{Gene Regulatory Inference via Linear Methods}

Gene regulatory network (GRN) inference is crucial for understanding cellular functions, developmental processes, and disease mechanisms due to their importance in governing cell identity \cite{Davidson2006-ql}. Inferring these networks from gene expression data is a challenging task due to the high-dimensional, noisy nature of the data, as well as the complex, nonlinear relationships between genes \cite{marbach2012}. Advancements in technology, particularly in the domain of single-cell transcriptomic data acquisition, have paved the way for the development of several innovative methodologies for GRN inference \cite{Pratapa2020, Nguyen2020ComprehensiveSurvey}. 

Historically, correlation-based methods have been used to determine the co-expression of gene patterns across samples \cite{Carter2004-xi, Babur2010-fg, Wang2009-na}. These methods are straightforward and fast to compute, but can only capture linear or monotonic dependencies and cannot distinguish between direct and indirect interactions \cite{butte2000} Furthermore, gene co-expression cannot capture causal information about how genes relate to one another.

Another common approach to gene regulatory inference is linear regression modeling, through describing each gene's expression level as a weighted sum of the expression of its putative regulators. Notably, most of these approaches add additional regularization to limit overfitting and distinguish active connections from inactive or false connections. Lasso regression \cite{tibshirani1996}, ridge regression \cite{Hoerl1970-yj}, LARS \cite{Efron2004-ma}, and elastic net \cite{zou2005} have all been used for GRN inference, including within toolboxes Inferelator \cite{Bonneau2006-px, Miraldi2019-ba}, CellOracle \cite{Kamimoto2023-yc}, and TIGRESS\cite{haury2012}, as well as a computational model of single-cell perturbations \cite{Dixit2016-gy}. Although these methods can capture both positive and negative interactions and can infer direct interactions by introducing sparsity in the inferred networks, they assume linear relationships between genes and struggle with scale.

\subsubsection{Information-theoretic approaches}

Non-linear dependencies can be captured by information-theoretic methods, such as maximum entropy \cite{Lezon2006-uz, Locasale2009-kx}, mutual information \cite{meyer2007,de2021latent}, and conditional mutual information \cite{Krishnaswamy2014-vj}, where one measures the mutual dependence between two genes' expression profiles. Partial information decomposition has also been used to consider triplets of genes within the networks \cite{Chan2017-ac}. Notably, such methods can infer direct interactions by conditioning on the expression of other genes \cite{meyer2007}; however, they can be computationally expensive and require a large number of samples to accurately estimate mutual information.



\subsection{Dynamics Modeling}

Efforts in dynamics modeling have also been influential in interaction graph inference. Prasse et al. presented a method that leverages a system of ordinary differential equations (ODEs) to predict the dynamics of complex systems, given an interaction graph \cite{prasse_predicting_2022}. These are designed coupled ordinary differential equation systems that can be used for simulation and data generation. Recently such frameworks have been applied in math biology to describing transitions in single-cell RNA-sequencing data via a concept known as {\em RNA velocity} \cite{bergen2020generalizing, chen2022deepvelo}.
A key insight in RiTINI is to co-opt the ODE based dynamics modeling framework for learning rather than just following a trajectory, and in the process also inferring the regulatory connections.

\section{Background}

\subsection{Regulatory networks and time series}

In biological systems, there exist complex regulatory interactions that are often characterized by proteins, such as transcription factors (TFs), that can stimulate a target gene (possibly another transcription factor) by binding to a nearby region of DNA, often referred to as a promoter \cite{bansal2007}. Ergo, this interaction between a transcription factor and target is physically achieved by binding and proximity. These interactions can be dynamic and influenced by numerous factors such as the availability of the regulatory protein, its translation and degradation rate, and its binding affinity to the target \cite{gardner2003}. Furthermore, epigenetic mechanisms, such as chromatin accessibility in the target region, can mediate these interactions. Other regulatory interactions may include enhancer-target interactions, where DNA loops facilitate the connection of two DNA regions, enabling other proteins that stimulate gene expression to bind. Gene regulatory effects can also be indirect via different links in the chain. In this context, these physical regulatory processes are abstracted by considering a time-varying graph inferred via a gene expression prediction task, regularized by a prior.

Neurons connect to each other via dynamic synaptic connections that can propagate activation signals from one neuron to another. If a source neuron fires, then depending on neurotransmitters and other molecules mediating connectivity, another neuron can also fire if the action potential is greater than the target. These synaptic connections are widely recognized to be both plastic and capable of saturation over time. When presented with such data, we can again abstract these physical processes by modeling a time-varying graph \cite{geweke1984}. These can also be abstracted as time-varying, potentially cyclic graphs whose connections give rise to signal patterns. However, in these cases, the prior network is not well-mapped out, therefore we use time-lagged Granger causality to define a starting network \cite{geweke1984}.


\subsection{Graph ODEs}

Neural ODEs represent an emerging paradigm that computes a derivative, instead of a function, parameterized by a neural network \cite{chen2018}. These derivatives are then integrated using an ODE solver to later timepoints, where they can be penalized by discretely collected time series data. Thus neural ODEs are useful for learning the dynamics of systems from data.  Graph ODEs are a more recent extension of neural ODEs, which use a graph neural network for the computation of the derivative \cite{poli2020}. 

Graphs naturally provide greater interpretability as they can encode the coupling structure of a dynamic system, which in this context, is the regulatory graph being inferred. Additionally, graphs, like transformers, can be endowed with attention mechanisms, which allow for the dynamic change in strength between two vertices \cite{velickovic2017}. This not only enhances the model's flexibility in handling complex systems but also augments its interpretability and ability to capture complex dependencies, which is a crucial aspect when dealing with high-dimensional data or systems with intricate interactions \cite{rackauckas2020}. Further facilitating graph ODEs ability to learn dynamics are their ability to handle continuous-time data and incorporate prior knowledge about the system, such as known interactions or network structures \cite{rackauckas2020, poli2020}. This is particularly useful in biological systems where prior knowledge from existing databases or literature can be used to guide the network inference process \cite{bergen2020generalizing, chen2022deepvelo}. 

\section{Methods}
First we describe the setup of our problem including formally defining an interaction graph. Next we present an architecture suited to this particular problem. 

\subsection{Problem Setup}

Our main goal in this work is to infer an interaction graph via the proxy task of matching time-trajectories of node features of the graph. As discussed previously, we take inspiration from biological networks. Here, we define this problem mathematically to highlight facets of such networks.  

\begin{definition}
We define an interaction graph as a directed time-varying graph $G(t)=(\mathcal{V}, \mathcal{E}(t))$ on a fixed set of vertices $\mathcal{V} = \{1,2,\cdots, N\}$, with dynamic edge weights parameterized by time $W_{ij}(t) > 0$ if $(i,j) \in \mathcal{E}(t)$. 
\end{definition}

We assume that features $X(v_i, t)$ on a vertex $v_i$ at time $t$ are a function $r$ of the features of its neighboring vertices $v_j \in \mathcal{N}(i)$ at recent time interval $[t - \tau, t]$ and itself. In cases of systems with no hysteresis this can be a single timepoint, $t-\delta \in [t - \tau, t]$, representing the lag of information flow between the two vertices. The function $r$ is assumed to be time-dependent, meaning that it can change over time. However, we implicitly assume it changes at longer time scales than $\tau$, i.e., that it is metastable over short time ranges.
\begin{align}
    X(v_i, t) = r \left( \{ X(v_j,t-\delta)~ |~ v_j \in \mathcal{N}(i)~ \cup~ \{i\} \} \right)
\end{align}

%


We also assume that there is prior knowledge available to regularize this graph, i.e deviation from the prior graph, $\mathcal{P}$, can be penalized by $ \mathcal{L}_{\mathcal{P}}(t) = \| W(t) - \mathcal{E}_{\mathcal{P}} \|_{F}$.

Further we assume that the graph we want to infer is sparse. To encourage sparsity, this could be explicitly modelled by penalizing the Frobenius norm, $\mathcal{L}_{\text{sparse}} = | W(t) |_1$.

Collectively, objective function, which we will translate to our neural network loss function is given below:
\begin{align}
\mathcal{L} = \mathcal{L}_{\text{MSE}} + \lambda_1 \mathcal{L}_{\mathcal{P}} + \lambda_2 \mathcal{L}_{\text{sparse}},
\label{eq:loss}
\end{align}
where, $\mathcal{L}_\text{MSE}$ is the mean squared error of the vertex features, $\lambda_1$ is a hyper-parameter requiring the inferred graph to be close to the prior and $\lambda_2$ is the weight of the sparsity regularization.







\subsection{RiTINI architecture}

Here we describe the architecture of RiTINI and relate architectural choices to facets of the problem setup. 
RiTINI is generally a graph ODE network, i.e., it computes derivatives of vertex features of a graph. As such, it consists of a single Graph Attention (GAT) layer that produces derivatives of vertex features, which are then integrated out to match timepoints. RiTINI starts with a prior graph, $\mathcal{P}$, as the initial graph, and then uses space attention to modify graph aggregation operations. RiTINI features two sets of attention parameters, one that controls edge strength between vertices which represent interactors, and another that controls the time lag or hysteresis within interactions of the system. 

Given high dimensional traces of data, each vertex of the graph $G(t)$, is an individual feature. For each node $v_i \in \mathcal{V}$ at time $t_0$, the space and time attention layer computes an aggregated representation $g'_i(t)$:
\begin{align}
\label{eqn:deriv}
    g_i'(t) &= \sum_{\delta \in [0, \tau]} l_\delta \cdot \Big( \alpha_{ii}(t)WX(v_i, t-\delta) + \sum_{j\in \mathcal{N}(i)} \alpha_{ij}(t) W X(v_j, t-\delta) \Big),
\end{align}
where $l_\delta$ is the attention over the time trace, that is able to enforce hysteresis, $\delta$ is the time-lag, $\alpha_{ij}(t)$ is the "space" attention coefficient between vertex $i$ and vertex $j$ at time $t$, $\mathcal{N}(i)$ is the set of neighbors of vertex $i$, and $W \in \mathbb{R}^{d \times D}$ is a learnable weight matrix that projects the $D$-dimensional input features to a lower dimension $d$. The attention coefficients are computed using a learnable attention mechanism:
%
%
\begin{align}
\label{eqn:attn}
\alpha_{ij}(t) = \text{Att}_\phi \Bigg( \bigcup_{\delta \in [0,\tau]} X(v_i,t-\delta) \circ X(v_j,t-\delta) \Bigg)
\end{align}
where $\circ$ denotes concatenation, and $\text{Att}$ is a feed-forward attention network parameterized by $\phi$ that computed the attention paid by vertex $i$ towards vertex $j$ over the time window $[t-\tau, t]$.

Next, we model the dynamics of the graph as a continuous-time process by leveraging Neural ODEs. The Neural ODE is a function $f: \mathbb{R}^d \times \mathbb{R} \rightarrow \mathbb{R}^d$ parameterized by a neural network. The ODE is defined as:
\begin{align}
\frac{d}{dt}g_i(t) = f_\theta(g_i'(t), t),
\label{eq:ODE}
\end{align}
where $\theta$ are the learnable parameters of the neural network, and $h_i(t)$ is the hidden representation of vertex $v_i$ at time $t$. We integrate this ODE from an initial time $t_0$ to a target time $t_1$ (the next time sample available in the training data) using an ODE solver:
\begin{align}
\label{eqn:odesolve}
\hat{X}(v_i, t) &= \text{ODESolve}(f_\theta, g'_i(t_0), t_0, t_1) \quad \forall t \in (t_0, t_1]
\end{align}
To train the network, we obtain observed or generated time traces from the dynamic system at hand. For each feature in the data (corresponding to a vertex in our graph) we apply a MSE loss to ensure prediction of the time trajectory. Additionally, we apply regularizations to encourage sparsity in the inferred edges of the graph via the space attention, $\alpha_{ij}$, by assuming that the inferred graph is not too distant from the prior, $\mathcal{P} = (\mathcal{V}_\mathcal{P}, \mathcal{E}_\mathcal{P})$. Optionally, we can also apply a $L_1$ regularization to the attention coefficients:
\begin{align}
\label{eqn:ritiniloss}
\mathcal{L}(t) &= \sum_i \| \hat{X}(v_i,t) - X(v_i,t) \| + \lambda_1 \sum_{i,j}  \| \alpha_{ij}(t) - \mathcal{E}_\mathcal{P} \|_F + \lambda_2 |\alpha(t)|_1
\end{align}
Note that RiTINI can give several useful outputs: 
\begin{itemize}
\item A dynamic graph, i.e. one graph per time point based on the attention  $\alpha_{ij}(t)$ between each pair at time $t$.
\item A static graph, $G = (\mathcal{V},\mathcal{E})$, that is the average of $\alpha_{ij}(t)$ over time 
\item Predicted time trajectories that can be interpolated to with-held timepoints that closely match the timepoints of the system. 
\end{itemize}

\begin{figure}[h]
    \centering
    \includegraphics[width=0.9\linewidth]{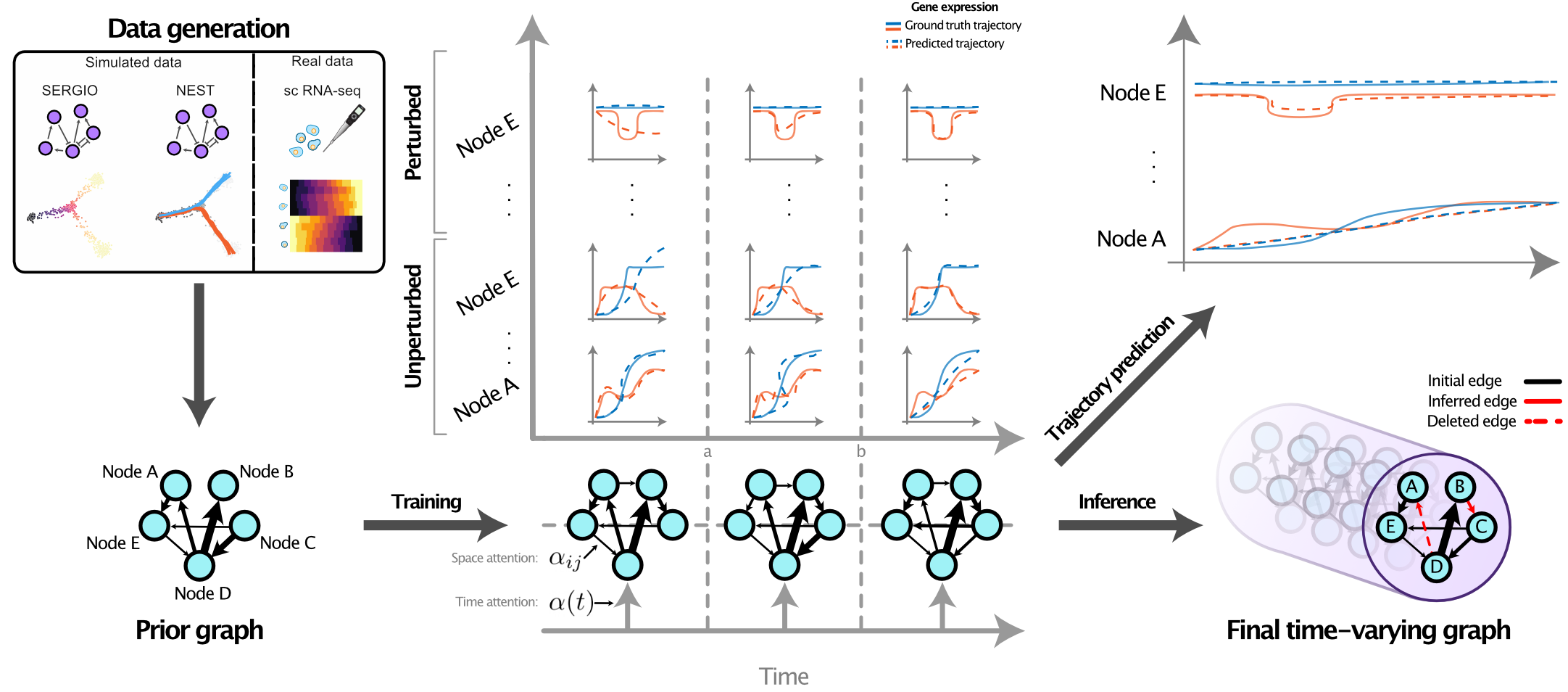}
    \caption{The RiTINI architecture takes in either simulated (e.g. from SERGIO or NEST) or real world (e.g. single cell RNA-seq) time series data and infers a prior graph if needed. Then the graph ODE is then trained as outlined in section \ref{sec:training} utilizing both space-and-time graph attentions. Thereafter, through inference, RiTINI produces the final time-varying graph, which can produce trajectory predictions for the dynamics of each node.}
    \label{fig:RiTINI_schematic}
\end{figure}
 
\subsection{Training with perturbation data}\label{sec:training}

In addition to training with time-trace data, RiTINI can be trained with perturbation data which is available in biological systems as optogenetic ablations in neuronal systems, or gene knockdowns or knockouts (CRISPR perturbations) in gene regulatory systems. Such data enhances learning of the network structure, as we show in the appendix.

\begin{definition}
We define a perturbation to vertex $v_i$ at  a time $t$, denoted $\Tilde{X}(v_i,t)$, as the trajectory obtained by add a small amount of noise, $\epsilon$, to the vertex signal at $v_i$ at time $t_p$, i.e., $\Tilde{X}(v_i,t_p) = X(v_i,t_p) + \epsilon$.
\end{definition}

We incorporate perturbations in the training of RiTINI to enhance its understanding of the network:
\begin{align}
\mathcal{L}_{\text{perturb}}(t) &= \sum_i  \| \hat{\Tilde{X}}(v_i,t) - \Tilde{X}(v_i,t) \|,
\end{align}
\begin{lemma}
Training the network to recover the perturbed dynamics, $\Tilde{X}(v_i,t)$, of vertex $v_i$ results in strengthening or weakening of attention on $v_i$ due to propagation of the signal to neighboring vertices $v_j$, i.e. $\{\alpha_{ij} \, | \, j \in \mathcal{N}(i) \}$.
\end{lemma}

\begin{proof}

Up to first order approximation, the predicted trajectory of vertex $v_i$ is given by:
\begin{align*}
\hat{\Tilde{X}}(v_i,t) = \int_{t_0}^t f_\theta(\Tilde{g}'_i(s), s) ds \approx \int_{t_0}^t \Big( f_\theta(g'_i(s),s) + J_f(g'(s))(\Tilde{g}'_i(s) - g'_i(s)) \Big) ds
\end{align*}
where the derivatives $\Tilde{g}_i$ and $g_i$ are estimated using Eq. \ref{eqn:deriv}. The Jacobian is given by:
\begin{align*}
\label{eqn:jacobian}
(J_f(g'(s)))_{ij} = \frac{\partial f_i(g_i(s))}{\partial g_j(s)} \propto \alpha_{ij}
\end{align*}
\end{proof}



\subsection{Static interaction graph inference using MAP estimation}

Since RiTINI outputs can be averaged over time to provide a static graph, we can show that based on our loss function construction, that in specific cases RiTINI performs maximum apriori estimation of the graph. 

Under the Bayesian paradigm, the static graph inference problem involves estimating a graph $\mathcal{G}^*$ as follows $\mathcal{G}^* = \argmax_{\mathcal{G}} \mathbb{P}(\mathcal{G}) \cdot \mathbb{P}(\mathcal{D} \mid \mathcal{G})$. $\mathbb{P}(G)$ represents a prior distribution on the space of regulatory graphs, $\mathcal{P} = (\mathcal{V}_\mathcal{P} , \mathcal{E}_\mathcal{P}) \sim \mathbb{P}(G)$ and $\mathbb{P}(\mathcal{D} \mid \mathcal{G})$ is the likelihood of the data. We show that our learning objective corresponds to a MAP objective with a specific graph prior.

\begin{prop}
The learning objective of Eq. \eqref{eq:loss} corresponds to a maximum a-posteriori objective where the data is assumed to be generated from a Gaussian process with white kernel and mean process following the system of ordinary differential equations of Eq. \eqref{eq:ODE}. The prior distribution is a Boltzman distribution centered at $\mathcal{P}$  where the distance between two graphs $\mathcal{G}$ and $\mathcal{G}'$ with adjacency matrices $A$ and $A'$ is given by $d(\mathcal{G},\mathcal{G}') = \alpha \lVert \mathcal{E}_\mathcal{G} - \mathcal{E}_{\mathcal{G}'}  \rVert_F + (1-\alpha) \lVert \mathcal{E}_\mathcal{G}  -\mathcal{E}_{\mathcal{G}'}  \rVert_1$, with $\alpha \in [0,1]$.
\end{prop}

The process mean $\mu_i(t)$ dynamics follow the same as Eq. \eqref{eq:ODE}:
\begin{align*}
    \frac{d\mu_{i}}{dt}(t) = f_{\theta}(\mu_{i}(t),t,\alpha)
\end{align*}
and the stochasticity is assumed to arise for an independent additive Gaussian noise $\epsilon \sim \mathcal{N}(0,\sigma)$.
Under these conditions, the log likelihood is proportional to 
\begin{align*}
    \log \mathbb{P}(\mathcal{D} \mid \mathcal{G}) \propto - \sum_i \lVert \mu_i (t) - X(v_i,t) \rVert_2^2,
\end{align*}
where we made the dependence on the interaction graph explicit. The graph prior is proportional to 
\begin{align*}
    \log \mathbb{P}(\mathcal{G}) \propto \alpha \lVert \mathcal{E}_\mathcal{G} - \mathcal{E}_\mathcal{P} \rVert_F +  (1-\alpha) \lVert \mathcal{E}_\mathcal{G} - \mathcal{E} \rVert_1,
\end{align*}

\section{Experiments}

In this section we showcase the performance of RiTINI in inferring interaction graphs by comparisons to other methods. We note that most other methods do not perform dynamic graph inference. For this purpose we used simulated datasets from NEST (a neuronal simulator) and SERGIO (a gene regulatory simulator) as well as on canonical dynamic systems. 


We performed simulations of a dynamical system, neuronal networks and gene regulatory networks to generate time-lapse data for training and validating the RiTINI architecture. All simulation and model training were performed on a Ubuntu 20.04.4 LTS machine with 8 NVIDIA Tesla K80 GPU accelerators. Source code is available at \href{https://anonymous.4open.science/r/RiTINI-7B44/}{https://anonymous.4open.science/r/RiTINI-7B44/}

First, we showcase the ability of RiTINI to follow time traces of a system, including the ability to interpolate to with-held timepoints in Appendix Section \ref{app: time-trace}. Next, we compare the performance of RiTINI in the graph inference task by comparing with other well known methods for static graph inference, including Granger Causality (GC) \cite{granger_investigating_1969}, Optimal Causation Entropy (OCE) \cite{sun_causal_2015}, a robust implementation of the PC-algorithm for DAGs \cite{Spirtes2000,kalisch_robustification_2008}, multivariate transfer entropy (mTE) \cite{schreiber_measuring_2000}, and multivariate mutual information (mMI) implemented in the IDTxl software library \cite{wollstadt_idtxl_2019} on dynamical systems, simulations of neuronal networks as well as gene regulatory inference as explained below. Explanations of these methods are given in Appendix Section \ref{app: other-methods}. 

\begin{table}[h]
    \setlength{\tabcolsep}{1.25pt}
    \centering
    \tiny
    \begin{tabular}{c c c c c c c c}
    \toprule
        \multicolumn{2}{c}{\textbf{Dataset}} & \multicolumn{6}{c}{\textbf{Method}} \\
        \cmidrule(lr){1-2}\cmidrule(lr){3-8}
         Simulation Type & Ground Truth & GC & OCE & PC & mTE & mMI & RiTINI (ours) \\
    \midrule
     \multirow{1}{*}{Dynamical System} & $|\mathcal{V}|=5,|\mathcal{E}|=5$ & $2.6 \pm 1.0$ & $4.4 \pm 1.6$ & $4.4 \pm 1.4$ & $4.6 \pm 1.7$ & $\mathbf{1.8 \pm 0.7}$ &  $2.2 \pm 1.6$\\
     \multirow{3}{*}{\shortstack{Neuronal Network (NEST)\\(exponential integrate-and-fire)}} & $|\mathcal{V}|=40,|\mathcal{E}|=78$ & $28.8 \pm 3.3$ & $72.8 \pm 3.5$ & $75.8 \pm 2.1$ & $64.2 \pm 3.5$ &  $\mathbf{20.6 \pm 5.0}$ & $25.4 \pm 2.8$\\
     & $|\mathcal{V}|=50,|\mathcal{E}|=126$ & $39.2 \pm 2.3$ & $109.6 \pm 4.2$ & $117.8 \pm 3.9$ & $100.0 \pm 7.6$ & $\mathbf{34.6 \pm 5.2}$ & $35.2 \pm 2.3$\\
     & $|\mathcal{V}|=75,|\mathcal{E}|=308$ & $75.4 \pm 9.2$ & $255.6 \pm 7.3$ & $284.6 \pm 3.2$ & $232.2 \pm 7.1$ & $82.0 \pm 3.9$ & $\mathbf{69.6 \pm 7.7}$\\
    \multirow{3}{*}{\shortstack{Gene Regulatory Network (SERGIO)\\(Michaelis–Menten)}} & $|\mathcal{V}|=100,|\mathcal{E}|=137$ & $51.2 \pm 3.3$ & $138.6 \pm 3.5$ & $140.4 \pm 3.9$ & $126.4 \pm 2.4$ & $51.2 \pm 3.3$ & $\mathbf{44.6 \pm 6.2}$\\
     & $|\mathcal{V}|=150,|\mathcal{E}|=329$ & $109.0 \pm 6.4$ & $293.4 \pm 2.9$ & $317.2 \pm 3.7$ & $261.0 \pm 2.2$ & $99.8 \pm 4.0$ & $\mathbf{83.6 \pm 4.2}$\\
     & $|\mathcal{V}|=200,|\mathcal{E}|=507$ & $158.8 \pm 12.6$ & $449.8 \pm 1.1$ & $495.6 \pm 6.5$ & $397.4 \pm 8.8$ & $162.8 \pm 6.2$ & $\mathbf{128.0 \pm 4.3}$\\
    \bottomrule
    \end{tabular}
    \caption{Mean and standard deviation of the graph edit distance between the inferred graph and the ground truth, across 5 different simulations with perturbations (lower is better).}
    \label{tab:comparison}
\end{table}

\begin{figure}[h]
    \centering
    \includegraphics[width=0.9\linewidth]{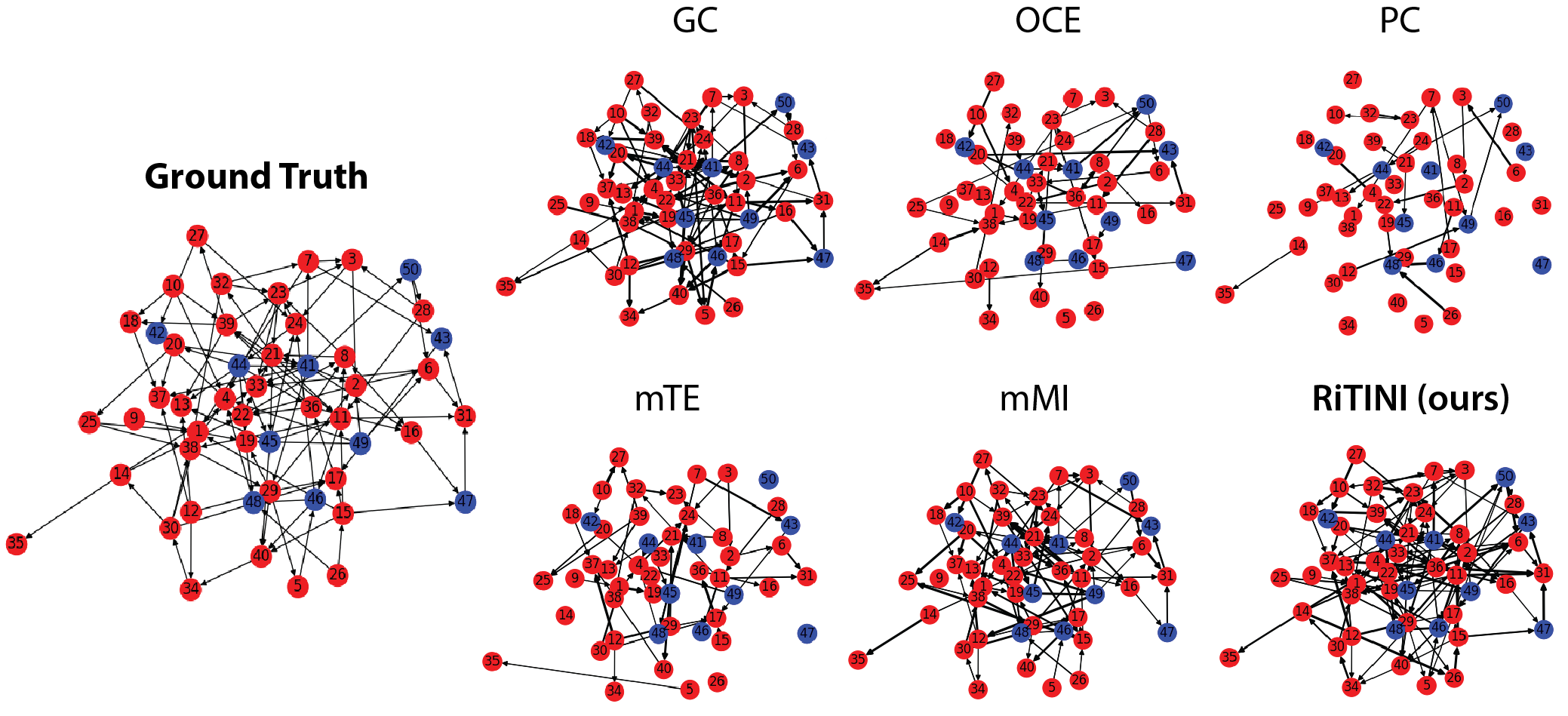}
    \caption{A neuronal network consisting of $50$ neurons (80\% excitatory labelled in red and 20\% inhibitory labelled in blue) was simulated using NEST \cite{sinha_ankur_2023_6867800} and the graphs inferred using the neuron firing rate as input features were compared with ground truth. Self-loops are not shown for ease of visualization.}
    \label{fig:NEST_inference}
\end{figure}

\subsection{Dynamical system}

We simulated a 5 node network, with the same structure as examples in \cite{baccala_partial_2001,prasse_predicting_2022}, by simulating the following non-linear system of ODEs:
\begin{align*}
    x_1(t) &= x_1(t) + 0.95 \sqrt{2} x_1(t-1) - 0.9025 x_1(t-2)\\
    x_2(t) &= x_2(t) + 0.5x_1^2(t-2) \\
    x_3(t) &= x_3(t) - 0.4x_1(t-3) \\
    x_4(t) &= x_4(t) - 0.5x_1^2(t-2) + 0.5\sqrt{2}x_4(t-1) + 0.25\sqrt{2}x_5(t-1) \\
    x_5(t) &= x_5(t) - 0.5\sqrt{2}x_4(t-1) + 0.5\sqrt{2}x_5(t-1)
\end{align*}

Additionally, we simulated this system with the addition of Gaussian white noise (with variance $\sigma$ ranging from $0.05$ to $0.25$) to obtain additional perturbed trajectories. The ground truth network for this system can be described by the set of vertices, $\mathcal{V} = \{x_1, x_2, \cdots, x_5 \}$, and edges, $\mathcal{E} = \{ x_1 \to x_2, x_1 \to x_3, x_1 \to x_4, x_4 \to x_5, x_5 \to x_4 \}$. We used RiTINI to predict the trajectory of this system at held-out timepoints and infer a static graph by time-averaging the learned attention coefficients (Appendix Section \ref{app: dyn}, Figure \ref{fig:fivenodeode}). On this small dynamical system, RiTINI was among the top 2 performers on the static graph inference task (Table \ref{tab:comparison}), alongside multivariate mutual information (mMI). In Appendix Section \ref{app: dyn}, we show that RiTINI outperforms all methods on larger dynamical systems with more than $10$ variables.

\subsection{Neuronal network}

NEST (Neural Simulation Tool), developed by Sinha et al. \cite{sinha_ankur_2023_6867800}, is a platform that offers a flexible and efficient approach for modeling and studying the behavior of various neuronal networks, particularly in the context of large-scale simulations. Using NEST, we generated random networks consisting of 40, 50 and 75 neurons using the leaky integrate-and-fire model with alpha-function kernel synaptic currents. We incorporated 80\% excitatory neurons and 20\% inhibitory neurons in our network. We used a multimeter device in the NEST simulator to record the membrane potential and firing rate of all neurons during the $1000$ ms simulations with a resolution of $0.1$ ms. Additionally, we systematically perturbed the parameters of the integrate-and-fire model, including the voltage threshold ($V_\text{th}$), resting membrane potential ($E_L$), capacity of the membrane ($C_m$) and the refractory period ($t_\text{ref}$) to obtain perturbed time-traces from the system. Further details of the model parameters, simulation parameters and perturbations are provided in Appendix Section \ref{app: nest}.

The node features obtained from the neuronal network simulation, including the perturbations, were used to train the RiTINI architecture. The time-averaged static graph inferred by RiTINI was compared against graphs obtained other methods by computing the graph edit distance to the ground truth (Table \ref{tab:comparison} and Figure \ref{fig:NEST_inference}).
RiTINI was among the top 2 performers on this task, along with multivariate mutual information (mMI). However, note the RiTINI is the only architecture capable of inferring a dynamic graph (Figure \ref{fig:timelapseattn} by leveraging learned space and time attention coefficients. The predicted time trajectories and interpolation at held-out time points obtained using RiTINI for these systems are provided in Appendix Section \ref{app: time-trace}.

\begin{figure}[h]
    \centering
    \includegraphics[width=0.9\linewidth]{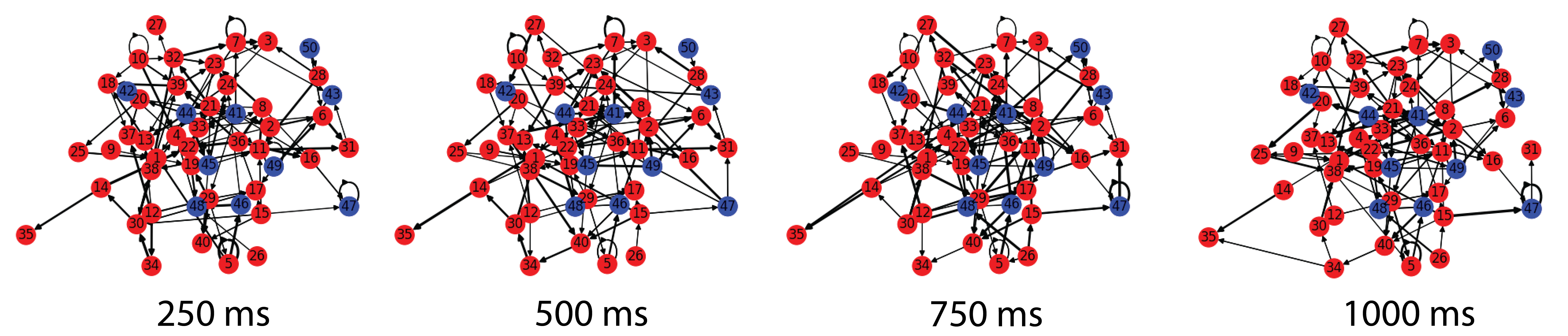}
    \caption{Dynamic edge weights inferred by thresholding learned attention coefficients over time, using trajectories from a NEST simulation consisting of 50 neurons (80\% excitatory labelled in red and 20\% inhibitory labelled in blue).}
    \label{fig:timelapseattn}
\end{figure}


\subsection{Gene regulatory network}

We used SERGIO (Single-Cell Expression Simulator Guided by Gene Regulatory Networks) \cite{dibaeinia_sergio_2020}, a tool developed by Dibaeinia and Sinha to simulate gene regulatory networks (GRNs). First, we defined a minimal network consisting of $5$ transcription factors that regulate cell differentiation. We then used SERGIO to build a gene regulatory network consisting of 100, 150 or 200 genes. We simulate this network to create time traces of a differentiation system with two branches. Additionally, we modified SERGIO to create perturbed time traces for training RiTINI. We compared the ability of RiTINI to infer gene regulatory interactions with other methods in Table \ref{tab:comparison} and found that it outperforms other methods. In the appendix we show that the attention-based inferred network moves from genes and transcription factors that are associated with the undifferentiated state to those that control differentiated states. 

\begin{figure}[h]
    \centering
    \includegraphics[width=0.6\linewidth]{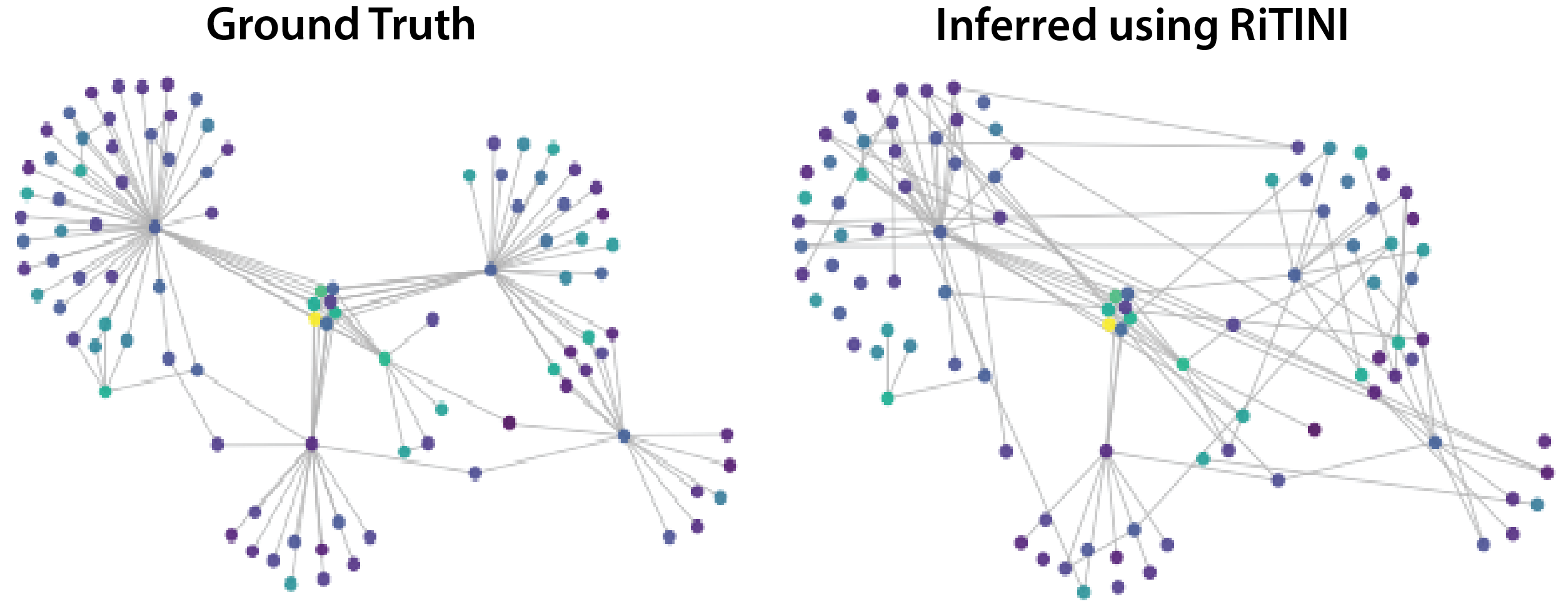}
    \caption{Ground truth GRN consisting of 100 genes and 137 edges from Table \ref{tab:comparison} and the static graph inferred using RiTINI by time-averaging the learned attention coefficients.}
    \label{fig:sergio}
\end{figure}

\section{Conclusion and Limitations}

We developed a novel method, RiTINI, designed to infer the underlying interaction graph from a multivariate time series datasets. This unique approach harnesses attention over space and time via  a graph ODE network, and operates in continuous time, to enable precise prediction and interpolation of the trajectories as well as inference of interaction graphs. Additionally, we show that our network can be trained with node perturbations for improved network inference. Notably, our method has proven to be competitive with state of the art methods across an array of real-world applications. For instance, in the fields of neurology, it has been effective in the recovery of neuronal networks, and in molecular biology, it has achieved state-of-the-art performance in the recovery of gene regulatory networks. 

One limitation of this work is that we do not model the inherent stochasticity of biological systems. For this, we could develop a version that uses SDE solvers, but would need to solve associated computational problems. 

\clearpage

\bibliographystyle{plainnat}
\bibliography{ref}

\clearpage

\appendix
\section{Interpolation of unperturbed and perturbed time traces using RiTINI}
\label{app: time-trace}

In the RiTINI architecture we leverage neural ODEs to predict time traces in both perturbed and unperturbed conditions. In Figure \ref{fig:timetrace} we demonstrate that our RiTINI architecture predicts time traces generated from NEST simulations of the neuronal network shown in Figure \ref{fig:NEST_inference} i.e. $50$ neurons (80\% excitatory labelled in red and 20\% inhibitory labelled in blue). In Figure \ref{fig:timetrace}, the ground truth traces obtained from the NEST simulations are plotted in blue. The neural ODE solution produced by RiTINI, in equally spaced training time intervals (denoted by red markers), is plotted in orange. These predicted time traces are in excellent agreement with the ground truth. Note that perturbations to the integrate-and-fire model parameters (described in Appendix Section \ref{app: nest} below) in neuron $1$ results in small deviations in its time trace, which propagate to the neighboring neuron $14$. However, no changes are observed in the time trace of neuron $5$, which is located far from the neighborhood of the perturbed neuron $1$. RiTINI correctly predicts the unperturbed and perturbed time traces of all neurons in this system. The accuracy of our model trained with and without perturbations can be found in Table \ref{tab:ablation}.


\begin{figure}[h]
\centering
\includegraphics[width=0.99\linewidth]{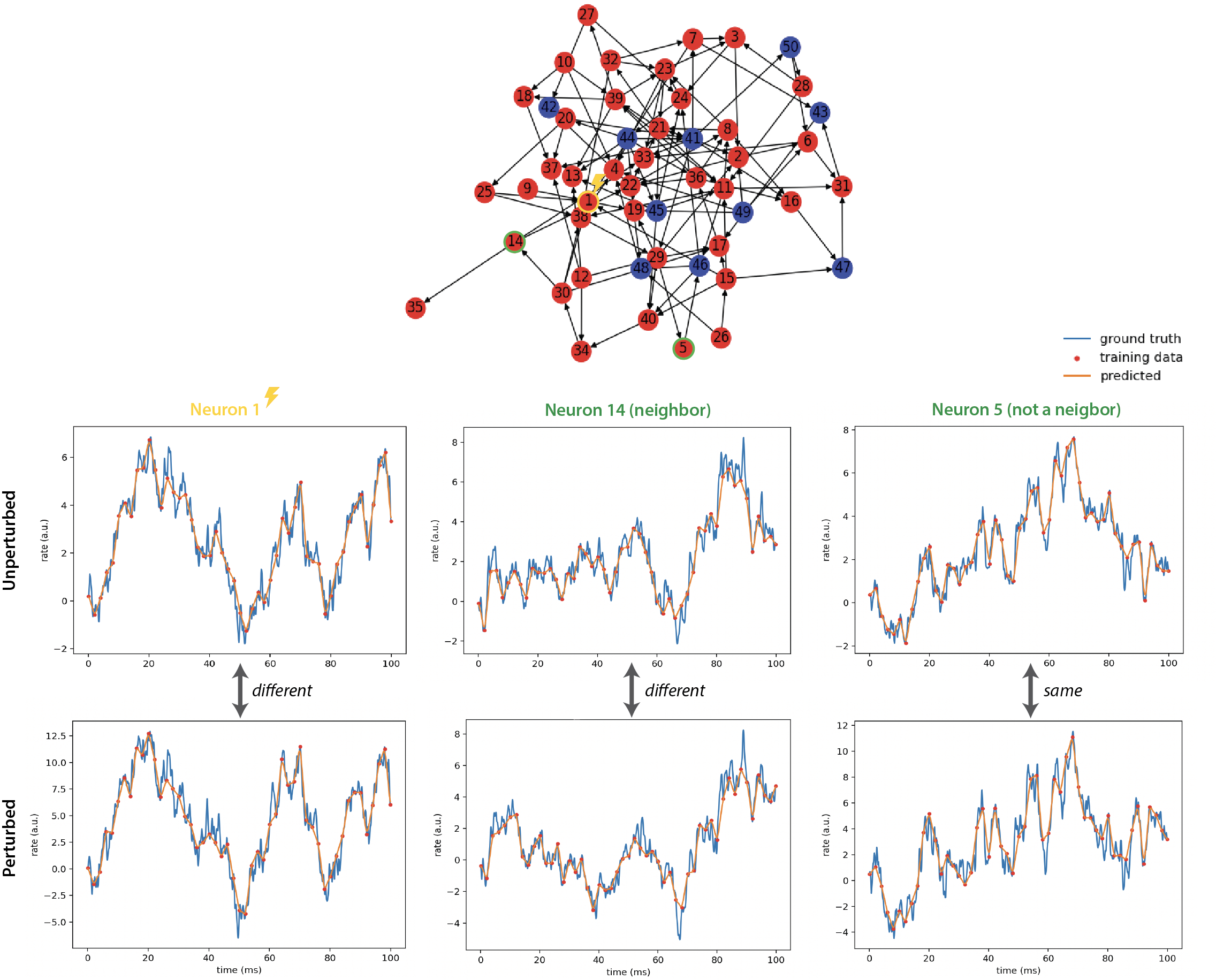}
\caption{Time traces generated from NEST simulation (ground truth) and RiTINI (predicted) of a neuronal network consisting of 50 neurons (same as Figure 2 in the main text) with and without perturbation applied to neuron 1 (highlighted in yellow)}
\label{fig:timetrace}
\end{figure}

\begin{table}[h]
\setlength{\tabcolsep}{1.25pt}
\centering
\scriptsize
\begin{tabular}{c c c c}
\toprule
\multicolumn{2}{c}{\textbf{Dataset}} & \multicolumn{2}{c}{\textbf{RiTINI}} \\
\cmidrule(lr){1-2}\cmidrule(lr){3-4}
Simulation Type & Ground Truth & Complete Model & Without Perturbations\\
\midrule
\multirow{3}{*}{\shortstack{Neuronal Network (NEST)\\(leaky integrate-and-fire)}} & $|\mathcal{V}|=40,|\mathcal{E}|=78$ & $25.4 \pm 2.8$ & $32.6 \pm 1.9$ \\
& $|\mathcal{V}|=50,|\mathcal{E}|=126$ & $35.2 \pm 2.3$ & $40.5 \pm 3.8$ \\
& $|\mathcal{V}|=75,|\mathcal{E}|=308$ & $\mathbf{69.6 \pm 7.7}$ & $80.1 \pm 5.7$ \\
\bottomrule
\end{tabular}
\caption{Mean and standard deviation of the graph edit distance between the inferred graph and the ground truth, across 3 different NEST simulations (lower is better).}
\label{tab:ablation}
\end{table}

\section{Description of baseline methods}
\label{app: other-methods}

\subsection{Granger Causality}

Granger causality is a statistical concept used to infer causality between variables in time series data \cite{granger_causality_1969}. The underlying principle of Granger causality is that if the past values of a variable $X$ can improve the prediction of another variable $Y$, then $X$ is said to ``Granger-cause'' $Y$. In other words, the inclusion of past values of $X$ provides additional information that enhances the prediction of $Y$. 

Mathematically, Granger causality can be expressed using autoregressive models (note that non-linear models can be considered in this context) \cite{wismuller_large-scale_2021}. Let $X_t$ represent the variable of interest at time $t$, and $X_{t-1}, X_{t-2}, \ldots, X_{t-p}$ denote its lagged values. Similarly, let $Y_t$ represent another variable at time $t$, and $Y_{t-1}, Y_{t-2}, \ldots, Y_{t-p}$ denote its lagged values. A linear autoregressive model for $X_t$ and $Y_t$ can be defined as:
\begin{align*}
X_t &= c_{X} + \sum_{i=1}^{p} \phi_{X,i} X_{t-i} + \sum_{i=1}^{p} \theta_{X,i} Y_{t-i} + \varepsilon_{X,t}\\
Y_t &= c_{Y} + \sum_{i=1}^{p} \phi_{Y,i} Y_{t-i} + \sum_{i=1}^{p} \theta_{Y,i} X_{t-i} + \varepsilon_{Y,t}
\end{align*}
where $c_X$, $c_Y$ are constants, $\phi_{X,i}$, $\phi_{Y,i}$ are coefficients for $X_{t-i}$ and $Y_{t-i}$ respectively, $\theta_{X,i}$, $\theta_{Y,i}$ are coefficients for the lagged values of the other variable, and $\varepsilon_{X,t}$, $\varepsilon_{Y,t}$ are error terms. According to this model, $X$ is said to "Granger-cause" $Y$ if any of the parameters $\theta_{X,i} \neq 0$.

In practice, one can assess the Granger causality from $X$ to $Y$, by comparing the predictive performance of two models: one including only the past values of $Y$ (null model), and another including the past values of both $X$ and $Y$ (full model). The improvement in prediction provided by the full model over the null model indicates Granger causality. 



\paragraph{Limitations} The fundamental prerequisite for Granger causality is separability, which entails that information regarding a causal factor is uniquely independent of the variable in question and can be eliminated by excluding that variable from the model. Separability is a distinguishing feature of linear and purely stochastic systems, and Granger causality can be valuable in identifying interactions among strongly coupled (synchronized) variables in nonlinear systems, but may fail when the coupling is weaker. The concept of separability aligns with the notion that systems can be comprehended incrementally, focusing on individual components rather than considering them as a whole \citep{sugihara2012detecting}.

\subsection{Optimal Causation Entropy}

The optimal causation entropy algorithm infers the causal network underlying the dynamics of a multivariate system by assessing the amount of information that flows from one variable to another. It builds upon the notion of (conditional) entropy of random variables.

The Shannon entropy of a continuous random variable $X$ is defined as
\begin{align}
h(X) = - \int p(x) \log p(x) dx,
\end{align}
where $p(x)$ is the probability density function of $X$.

The conditional entropy of $X$ conditioned on $Y$ is defined as 
\begin{align}
h(X \mid Y) = - \int p(x,y) \log p(x\mid y) dxdy,
\end{align}

\begin{definition}
The causation entropy from the set of nodes $J$ to the set of nodes $I$ conditioning on the set of nodes $K$ is defined as
\begin{align}
C_{J \to I | K} = h(X^{(I)}_{t+1}|X^{(K)}_t) -  h(X^{(I)}_{t+1}|X^{(K)}_t, X^{(J)}_t)
\end{align}
\end{definition}

The causation entropy is a generalization of transfer entropy that allows to condition to a set of variables. Intuitively, $C_{J \to I | K}$, represents the amount of information that the set of variables $J$ contributes to the set of variables $I$, conditioned on the set of variables $K$. A large value therefore indicates that the set of variables $J$ contains relevant information about $I$ that is not contained in $K$.

Using this definition, the optimal causation entropy algorithm proceeds in two steps. In the first step, a conservative set of causal parents for each node of the graph is constructed using the \emph{aggregative discovery of causal nodes} scheme. In the second step, the set of parents is pruned using the \emph{divisive removal of causal nodes} scheme.

\paragraph{Aggregative discovery of causal nodes.} In this step, for each node of the graph, we grow the set of potential parents by starting from empty sets $K$ and $J$. One then adds elements to $K$ by setting $K\leftarrow K \cup {j}$, where $j = \argmax C_{j \to I | K}$, for all $j \notin K$. The set of parents is grown until the  $C_{j \to I | K} \leq 0$, for all $j \notin K$.

\paragraph{Divisive removal of causal nodes.} The set of parents obtained from the first step is overly conservative. One then prunes this step to recover a better estimate of the list of causal parents of each node. For each node, we examine the elements of the set $K$ and remove all nodes $j$ such that which $C_{j \to I | K-\{j\}} = 0$.

\paragraph{Limitations.} The optimal causation entropy algorithm assumes the process is Markovian. That is the value of a variable at time $t$ can only depend on the causal parent at time $t-1$. In contrast, our approach allows to have different delayed representations of the signal impacting the dynamics of the system.

\subsection{Multivariate transfer entropy}

Like causation entropy, the transfer entropy is a non-parametric method that allows to quantify the asymmetric flow of information from one variable to another in the system. The transfer entropy from node $i$ to node $j$ is given by
\begin{align*}
T_{i\rightarrow j} = h(X_{t+1}^j \mid X_t^j) - h(X_{t+1}^j \mid X_t^j, X_t^i)
\end{align*}

The conditional transfer entropy, conditioned on a set of nodes $Z$ is given by 
\begin{align*}
T_{i\rightarrow j \mid Z} = h(X_{t+1}^j \mid X_t^j, Z) - h(X_{t+1}^j \mid X_t^j, X_t^i,Z)
\end{align*}

The algorithm to build the causal graph using the conditional transfer entropy proceeds in a greedy fashion. For each node $i$, we start with an empty set of causal parents $K$. One then adds elements to $K$ by setting $K\leftarrow K \cup {j}$, where $j = \argmax T_{j \to i | K}$, for all $j \notin K$. As the transfer entropy is monotonic in the number of elements in $K$, a stopping criterion is defined to ensure only \emph{significant} increases in the conditional transfer entropy are considered. The algorithm stops when no new nodes can significantly contribute to increasing the transfer entropy.

\paragraph{Limitations.} Being a non-parametric method, transfer entropy alleviates the problem of model misspecification. However, transfer entropy requires a lot of samples to be estimated reliably, which limits the applicability of this approach when only short time series are observed, which is common in biological applications.

\subsection{Multivariate mutual information}

Another variant of the above consists in using a normalized version of the transfer entropy:
\begin{align*}
C_{j \rightarrow i} = 1- \frac{h(X_{t+1}^i\mid \{ X_{t}^k , \forall k \in \mathcal{V} \})}{h(X_{t+1}^i\mid \{ X_{t}^k , \forall k \in \mathcal{V} -\{ j \} \})}
\end{align*}

This quantity is comprised between $0$ and $1$ and will be $0$ when variable $j$ does not contribute to variable $i$, and $1$ when variable $j$ completely determines variable $i$. One can extend this definition for a set of parent variables $K$:
\begin{align*}
C_{K \rightarrow i} = 1- \frac{h(X_{t+1}^i\mid \{ X_{t}^k , \forall k \in \mathcal{V} \})}{h(X_{t+1}^i\mid \{ X_{t}^k , \forall k \in \mathcal{V} - K \})}
\end{align*}

This quantity can then be used to construct the set of causal parent for each node $i$ in a greedy fashion, similarly as for the transfer entropy case.

\paragraph{Limitations.} Similarly as for the transfer entropy, reliable estimation of the mutual information requires a large number of samples, which limits the applicability of the method in biological applications.

\subsection{PC-algorithm}

The PC algorithm operates by finding conditional independence relations between pairs of variables. It starts by initialize an empty fully connected graph on all variables. It then reduces the fully connected graph to a skeleton graph by performing pairwise independence tests. If nodes $i$ and $j$ are independent, the edge between node $i$ and $j$ is removed. 

One the pairwise independence tests have been performed for all pairs of variables, one can further prune the undirected skeleton graph by using conditional independence test. For each pair of variables $(i,j)$, we test $X^i \perp X^j \mid Z$ for all sets $Z$ that contain neighbours of $i$ and $j$ and whose cardinality is lower than some threshold $d$. If a non-empty set $Z$ exists such that $X^i \perp X^j \mid Z$, the edge between $i$ and $j$ is removed. The separating set of $(i,j)$ is the largest set $Z$ that leads to independence. 

Lastly, the undirected skeleton is oriented to give the final causal graph. We first make a bidirectional graph by creating two directional edges between $i$ and $j$ if an edge exists in the skeleton. The orientation of the edge then happens according to 4 rules.

\begin{itemize}
\item (v-structures orientation): $i \leftrightarrow j \leftrightarrow k$ becomes $i \rightarrow  j \leftarrow k$ if $j$ is not in the separating set of $(i,k)$.
\item (new v-structures prevention): $ i \rightarrow j\leftrightarrow k$ becomes $i \rightarrow j \rightarrow k$ if $i$ and $k$ are not adjacent.
\item  (avoiding cycles): $i \rightarrow j  \rightarrow k \leftrightarrow i$ becomes $i \rightarrow j \rightarrow k \leftarrow i$.
\item  (combination): To avoid creating cycles or new v-structures, whenever $ i - j \rightarrow k$, $i - l \rightarrow k$, and $i - k$ but there is no edge between $j$ and $l$, the undirected $i - k$ edge becomes a directed edge $i \rightarrow k$.
\end{itemize}

\paragraph{Limitations.} The PC algorithm relies on the ability to effectively perform conditional independence tests. While these tests can be efficiently performed for binary data, they are notoriously difficult for continuous data. In fact the power of these tests is directly proportional to the number of samples available. The PC algorithm is thus not best suited for time series with few observations over time.

\section{Recovery of connectivity in dynamical systems}
\label{app: dyn}

Here we demonstrate that RiTINI can infer the relationship between variables in larger dynamical systems. We consider two such systems: a minimal model of collective firing in neuronal populations from the \emph{C. elegans} worm, and connectivity between cortical areas of the brain in macaque monkeys.

\begin{figure}[h]
    \centering
    \includegraphics[width=0.7\linewidth]{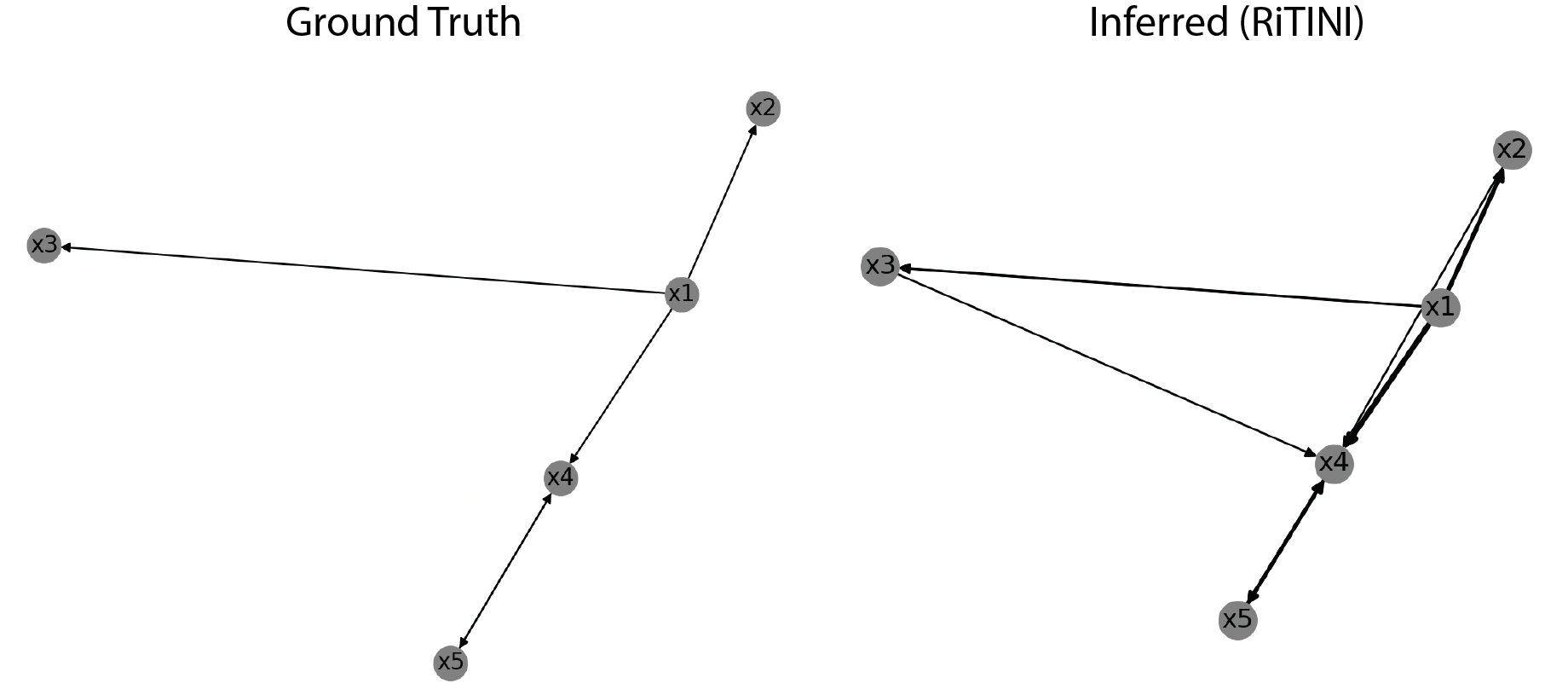}
    \caption{Static graph inferred from a nonlinear system of 5 coupled ODEs by time-averaging the learned attention coefficients.}
    \label{fig:fivenodeode}
\end{figure}

\paragraph{Wilson-Cowan model simulations of neural connectivity in \emph{C. elegans}}
The Wilson-Cowan model \cite{wilson1972excitatory, laurence2019spectral} is a computational framework widely used in neuroscience and theoretical biology to study the dynamics of large-scale neural populations. It provides a simplified description of the interactions between excitatory ($E$) and inhibitory ($I$) neurons in a network. The model assumes that the activity of a neural population can be represented by two variables: the average firing rate of excitatory neurons ($r_E$) and the average firing rate of inhibitory neurons ($r_I$). These variables are governed by a set of coupled differential equations that capture the balance between excitation and inhibition. The dynamics of the Wilson-Cowan model can be described as follows:
\begin{align*}
\frac{{dr_E}}{{dt}} &= \left( \alpha_E \cdot S_E(w_{EE} \cdot r_E - w_{EI} \cdot r_I) - \theta_E - r_E \right) / \tau_E \\
\frac{{dr_I}}{{dt}} &= \left( \alpha_I \cdot S_I(w_{IE} \cdot r_E - w_{II} \cdot r_I) - \theta_I - r_I \right) / \tau_I 
\end{align*}
where $\alpha_E$ and $\alpha_I$ represent the gain functions for excitatory and inhibitory neurons, $w_{EE}$, $w_{EI}$, $w_{IE}$, and $w_{II}$ are the synaptic weights, $\theta_E$ and $\theta_I$ are the threshold values, and $\tau_E$ and $\tau_I$ are the time constants for excitatory and inhibitory neurons, respectively. The functions $S_E$ and $S_I$ determine the activation response of excitatory and inhibitory neurons, which can be modeled using various functions such as sigmoidal or linear functions. The Wilson-Cowan model provides insights into the collective behavior of neural circuits, allowing researchers to study phenomena such as spontaneous oscillations, stable fixed points, and the impact of network parameters on neural dynamics. It serves as a valuable tool for understanding the complex interactions within neural populations and contributes to our understanding of brain function and information processing.

\begin{table}[h]
\setlength{\tabcolsep}{1.25pt}
\centering
\tiny
\begin{tabular}{c c c c c c c c}
\toprule
\multicolumn{2}{c}{\textbf{Dataset}} & \multicolumn{6}{c}{\textbf{Method}} \\
\cmidrule(lr){1-2}\cmidrule(lr){3-8}
Simulation Type & Ground Truth & GC & OCE & PC & mTE & mMI & RiTINI (ours) \\
\midrule
\multirow{1}{*}{Dynamic Mean-Field Model (DMF)} & $|\mathcal{V}|=82,|\mathcal{E}|=1560$ & $63.7 \pm 2.4$ & $85.9 \pm 6.8$ & $97.6 \pm 8.7$ & $61.5 \pm 3.8$ & $73.4 \pm 5.4$ &  $\mathbf{57.4 \pm 2.6}$\\
\multirow{1}{*}{Wilson-Cowan Neural Mass Model} & $|\mathcal{V}|=282,|\mathcal{E}|=2994$ & $95.0 \pm 8.8$ & $202.6 \pm 11.3$ & $194.8 \pm 5.4$ & $98.3 \pm 7.9$ & $107.5 \pm 9.6$ &  $\mathbf{86.8 \pm 4.1}$\\
\bottomrule
\end{tabular}
\caption{Mean and standard deviation of the graph edit distance between the inferred graph and the ground truth in dynamical system simulations (lower is better).}
\label{tab:largedyn}
\end{table}


\paragraph{Dynamical mean-field simulations of brain connectivity in macaques}
The Dynamic Mean-Field (DMF) model, as introduced by Deco and Jirsa (2012) \cite{deco2012ongoing}, is a framework used to understand the ongoing cortical activity in the brain at rest . The model is based on the idea that the brain operates in a state of criticality, which means it is poised at the edge of a phase transition between order and disorder. This critical state allows the brain to balance stability and flexibility, enabling it to respond to a wide range of inputs in a robust yet adaptable manner.

Notably the DMF model also introduces the concept of ``ghost attractors'', states that the brain can transition to but does not remain in for extended periods. Ghost attractors represent potential patterns of activity that the brain can access if needed, providing a reservoir of functional states that can be called upon in response to changing environmental demands.


The dynamic mean-field (DMF) model is a reduction of a spiking attractor network that consists of integrate-and-fire neurons with excitatory (NMDA) and inhibitory (GABA-A) synaptic receptor types. The neurons are organized into an inhibitory population (20\% of the neurons) and an excitatory population (80\% of the neurons).  DMF reduces this complex system by averaging the activity of large groups of neurons together, rather than trying to simulate each neuron individually, the global brain dynamics of which can be described by a set of coupled differential equations:
\begin{align*}
\frac{d S_i (t)}{d t} &= - \frac{S_i}{\tau_s} + \gamma (1 - S_i) H(x_i) + \sigma \nu_i(t) \\
H(x_i) &= \frac{ax_i - b}{1 - \exp^{- d (a x_i - b)}}\\
\frac{d x_i}{d t} &= w J_N S_i - G J_N \sum_{j} C_{ij} S_j + I_o     
\end{align*}
as detailed in Deco et al (2014) \cite{deco2014identification}: \(H(x_i)\) and \(S_i\) denote the population rate and the average synaptic gating variable at the local cortical area \(i\), \(w = 0.9\) is the local excitatory recurrence, \(G\) is a global scaling parameter, and \(C_{ij}\) is a connectivity matrix for neuroanatomical links between the cortical areas \(i\) and \(j\) . The connectivity matrix, $C_{ij}$, comes from 82 parcellated cortical areas extracted from macaque monkeys adopted in the CoCoMac-based non symmetric structural connectivity (SC) matrix \cite{bakker2012cocomac, shen2012information}.





\section{NEST Simulations}
\label{app: nest}

The leaky integrate-and-fire (IAF) model, as implemented in the NEST (Neural Simulation Tool) toolkit (denoted \texttt{iaf\_psc\_alpha} in the NEST documentation), provides a computationally efficient method to simulate the behavior of neurons in a network \cite{sinha_ankur_2023_6867800}. Here, the IAF model integrates input signals and generates an output when the integrated signal reaches a threshold. Mathematically, it is described by a differential equation,
\begin{equation}
\tau_m \frac{\delta V_m}{\delta t} = -(V_m - E_L) + \frac{I_{e}}{C_m},
\end{equation}
where \(V_m\) is the membrane potential, \(\tau_m\) is the membrane time constant, \(E_L\) is the resting membrane potential, \(I_{e}\) is the synaptic current, and \(C_m\) is the membrane capacitance. For a full list of settable parameters, see Table \ref{tab: lif-parameters}.

\begin{table}[h]
\begin{centering}

\begin{tabular}{lll}
\textbf{Parameter}        & \textbf{Units} & \textbf{Description}                                         \\
$V_m$            & mV    & Membrane potential                                  \\
$E_L$            & mV    & Resting membrane potential                          \\
$C_m$            & pF    & Capacity of the membrane                            \\
$\tau_m$         & ms    & Membrane time constant                              \\
$t_ref$          & ms    & Duration of refractory period                       \\
$V_{th}$         & mV    & Spike threshold                                     \\
$V_{reset}$      & mV    & Reset potential of the membrane                     \\
$\tau_{syn\_ex}$ & ms    & Rise time of the excitatory synaptic alpha function \\
$\tau_{syn\_in}$ & ms    & Rise time of the inhibitory synaptic alpha function \\
$I_e$            & pA    & Constant input current                              \\
$V_{min}$        & mV    & Absolute lower value for the membrane potential     
\end{tabular}
\caption{Settable parameters for the \texttt{iaf\_psc\_alpha} model from NEST.}\label{tab: lif-parameters}
\end{centering}
\end{table}

Although this IAF model in NEST neglects certain biological specifics, such as the exact shape of the action potential or various ion channel effects, its abstraction level allows for an efficient simulation of large neural networks.

\end{document}